\documentclass[11pt]{article}

\usepackage{color}
\usepackage{amsmath}
\usepackage{amssymb}
\usepackage{graphicx}
\usepackage{comment,xspace}
\usepackage{fancybox}

\newcommand{\real}{{\mathbb{R}}}
\newcommand{\reals}{\real}

\newcommand{\eps}{\varepsilon}

\newcommand{\TSP}{\ensuremath{\operatorname{TSP}}}

\newcommand{\card}{\operatorname{card}}

\newcommand{\expectation}[1]{\mbox{$\mathbb{E}\left[#1\right]$}} 
 
\newcommand{\condexpectation}[2]{\mbox{$\mathbb{E}\left(#1| #2\right)$}}

\pdfoutput=1
\usepackage{graphicx,psfrag}
\usepackage[noend]{algorithmic}
\usepackage[vlined,ruled,linesnumbered]{algorithm2e}
\SetKwInput{KwAssumes}{Assumes}
\usepackage{xspace}
\usepackage{amssymb,amsmath,amscd,mathrsfs, amsthm}
\usepackage{comment,color}
\usepackage{subfigure}
\usepackage{url}
\usepackage{multirow}
\usepackage{tikz}

\newtheorem{theorem}{Theorem}[section]

\newtheorem{lemma}[theorem]{Lemma}

\newtheorem{remark}[theorem]{Remark}

\newcommand{\R}{\mathbb{R}}

\newcommand{\env}{\mathcal{E}}
\newcommand{\btsp}{\beta_{\mathrm{TSP}}}

\newcommand{\upbd}{\textup{upbd}}
\newcommand{\opt}{\textup{opt}}

\newcommand\oprocendsymbol{\hbox{$\square$}}
\newcommand\oprocend{\relax\ifmmode\else\unskip\hfill\fi\oprocendsymbol}

\graphicspath{{./fig/}}

\title{{\bf Dynamic Multi-Vehicle Routing\\ with Multiple Classes of
  Demands}\thanks{This research was partially supported by the National
    Science Foundation, through grants \#0705451 and \#0705453, by the
    Office of Naval Research through grant \#N00014-07-1-0721, and by
    the Air Force Office of Scientific Research through grant
    \#FA9550-07-1-0528.}}

\author{Marco Pavone \quad Stephen L. Smith \quad Francesco Bullo
  \quad Emilio Frazzoli%
  \thanks{M. Pavone and E. Frazzoli are with the Laboratory for
    Information and Decision Systems, Aeronautics and Astronautics
    Department, Massachusetts Institute of Technology, Cambridge, MA
    02139, USA; email: {\ttfamily \{pavone,frazzoli\}@mit.edu}.
    S. L. Smith and F. Bullo are with the Center for Control,
    Dynamical Systems and Computation, Department of Mechanical
    Engineering, University of California, Santa Barbara, CA 93106,
    USA \texttt{\{stephen,bullo\}@engineering.ucsb.edu}.  } }

\begin{document}
\maketitle

\begin{abstract}
  In this paper we study a dynamic vehicle routing problem in which
  there are multiple vehicles and multiple classes of demands.
  Demands of each class arrive in the environment randomly over time
  and require a random amount of on-site service that is
  characteristic of the class.  To service a demand, one of the
  vehicles must travel to the demand location and remain there for the
  required on-site service time.  The quality of service provided to
  each class is given by the expected delay between the arrival of a
  demand in the class, and that demand's service completion.  The goal
  is to design a routing policy for the service vehicles which
  minimizes a convex combination of the delays for each class. First,
  we provide a lower bound on the achievable values of the convex
  combination of delays. Then, we propose a novel routing policy and
  analyze its performance under heavy load conditions (i.e., when the
  fraction of time the service vehicles spend performing on-site
  service approaches one). The policy performs within a constant
  factor of the lower bound (and thus the optimal), where the constant
  depends only on the number of classes, and is independent of the
  number of vehicles, the arrival rates of demands, the on-site
  service times, and the convex combination coefficients.

\end{abstract}

\section{Introduction}

Consider a bounded environment $\env$ in the plane which contains $n$
service vehicles.  Demands for service arrive in $\env$ sequentially
over time and each demand is a member of one of $m$ classes.  Upon
arrival, a demand assumes a location in $\env$, and requires a class
dependent amount of on-site service time.  To service a demand, one of
the $n$ vehicles must travel to the demand location and perform the
on-site service.  If we specify a policy by which the vehicles serve
demands, then the expected delay for demands of class $\alpha$,
denoted $D_{\alpha}$, is the expected amount of time between a demands
arrival and its service completion.  Then, given coefficients
$c_1,\ldots,c_m >0$, the goal is to find the vehicle routing policy
that minimizes 
\[
c_1D_{1} + \cdots + c_m D_m.
\]
By increasing the coefficients for certain classes, a higher priority
level can be given to their demands.  This problem, which we call
\emph{dynamic vehicle routing with priority classes}, has important
applications in areas such as UAV surveillance, where targets are
given different priority levels based on their urgency or potential
importance.

In classical queuing theory (i.e., queuing systems in which the
demands are not spatially distributed), the problem of priority queues
has received much attention,~\cite{LK:76}.  In~\cite{EGC-IM:80} the
authors characterize the region of delays that are realizable by a
single server.  This analysis is performed under the assumption that
the customer (demand) interarrival times and service times are
distributed exponentially.  In~\cite{DB-ICP-JNP:94} the achievable
delays are studied in more a general setting known as queuing
networks.

If service demands are spatially distributed, then providing service
becomes a problem in dynamic vehicle routing (DVR).  One of the first
DVR problems was the dynamic traveling repairperson problem
(DTRP)~\cite{DJS-GJvR:91,DJS-GJvR:93a}.  The DTRP is the single class
version of the dynamic vehicle routing with priority classes problem
studied in this paper.  In~\cite{DJS-GJvR:91,DJS-GJvR:93a}, the
authors study the expected delay of demands and propose
optimal policies in both heavy load (i.e., when the fraction of time
the service vehicles spend performing on-site service approaches one),
and in light load (i.e., when the fraction of time the service
vehicles spends performing on-site service approaches
zero).  
In \cite{EF-FB:03r}, and \cite{MP-EF-FB:07g}, decentralized policies
are developed for the DTRP.  Spatial queuing problems have also been
studied in the context of urban operations research~\cite{RCL-ARO:81},
where approximations are used to cast the problems in the traditional
queuing framework.  In our previous paper~\cite{SLS-MP-FB-EF:08g}, we
introduced and studied dynamic vehicle routing with priority classes,
for the case of two classes and one vehicle.  For this case we derived
a lower bound on the achievable delay values and proposed the
Randomized Priority policy, which performed within a constant factor
of the lower bound, for all convex combination coefficients.

The contributions of this paper are as follows.  We extend the dynamic
vehicle routing with priority classes problem to $n$ service vehicles
and $m$ classes of demands.  The extension of our previous analysis to
multiple classes of demands is very nontrivial. We derive a new lower
bound on the achievable values of the convex combination of delays,
and propose a new policy in which each class of demands is served
separately from the others.  We show that the policy performs with a
constant factor of $2m^2$ of the optimal.  Thus, the constant factor
is independent of the number of vehicles, the arrival rates of
demands, the on-site service times, and the convex combination
coefficients.  We also comment on the source of the gap between the
upper and lower bounds.

The paper is organized as follows.  In Section~\ref{sec:background} we
give some asymptotic properties of the traveling salesperson tour.  In
Section~\ref{sec:prob_stat} we formalize the problem and in
Section~\ref{sec:lower_bd} we derive a lower bound, and in
Section~\ref{sec:SQ_policy} we introduce and analyze the Separate
Queues policy.  Finally, in Section~\ref{sec:simu} we present
simulation results.
   
\section{Background and Problem Statement}
\label{sec:background}

In this section we summarize the asymptotic properties of the
Euclidean traveling salesperson tour, and formalize dynamic vehicle
routing with priority classes.

\subsection{The Euclidean Traveling
  Salesperson Problem}

Given a set $Q$ of $N$ points in $\R^2$, the Euclidean traveling
salesperson problem (TSP) is to find the minimum-length tour of $Q$
(i.e., the shortest closed path through all points). Let $\TSP(Q)$
denote the minimum length of a tour through all the points in $Q$.
Assume that the locations of the $N$ points are random variables
independently and identically distributed, uniformly in a compact set
$\env$ with area $|\env|$; in \cite{JMS:90} it is shown that there
exists a constant $\beta_{\mathrm{TSP}}$ such that, almost surely,
\begin{equation}
\label{eq:tspd}
\lim_{N\rightarrow+\infty} \frac{\TSP(Q)}{\sqrt{N}} =
\beta_{\mathrm{TSP}} \sqrt{|\env|}.
\end{equation}
The constant $\btsp$ has been estimated numerically as $\btsp \approx
0.7120\pm 0.0002$,~\cite{GP-OCM:96}.  The bound in
equation~\eqref{eq:tspd} holds for all compact sets $\env$, and the
shape of $\env$ only affects the convergence rate to the limit.
In~\cite{RCL-ARO:81}, the authors note that if $\env$ is ``fairly
compact [square] and fairly convex'', then equation~\eqref{eq:tspd}
provides an adequate estimate of the optimal TSP tour length for
values of $N$ as low as 15.


\subsection{Problem Statement}
\label{sec:prob_stat}

Consider a compact environment $\env$ in the plane with area $|\env|$.
The environment contains $n$ vehicles, each with maximum speed $v$.
Demands of type $\alpha\in\{1,\ldots,m\}$ (also called
$\alpha$-demands) arrive in the environment according to a Poisson
process with rate $\lambda_{\alpha}$.  Upon arrival, demands assume an
independently and uniformly distributed location in $\env$.  An
$\alpha$-demand is serviced when the vehicle spends an on-site service
time at the demand location, which is generally distributed with
finite mean~$\bar s_{\alpha}$.


Consider the arrival of the $i$th $\alpha$-demand.  The service delay for
the $i$th demand, $D_{\alpha}(i)$, is the time elapsed between its arrival
and its service completion.  The wait time is defined as $W_{\alpha}(i) :=
D_{\alpha}(i) - s_{\alpha}(i)$, where $s_{\alpha}(i)$ is the on-site
service time required by demand $i$. A policy for routing the vehicles is
said to be \emph{stable} if the expected number of demands in the system
for each class is bounded uniformly at all times.  A necessary condition
for the existence of a stable policy is
\begin{equation}
\label{eq:rho_def}
\varrho := \frac{1}{n}\sum_{\alpha=1}^m \lambda_{\alpha} \bar s_{\alpha} <1.
\end{equation}
The \emph{load factor} $\varrho$ is a standard quantity in queueing
theory~\cite{LK:76}, and is used to capture the fraction of time the
$n$ servers (vehicles) must be busy in any stable policy. In general,
it is difficult to study a queueing system for all values of
$\varrho\in[0,1)$, and a common technique is to focus on the limiting
regimes of $\varrho\to1^-$, referred to as the \emph{heavy-load}
regime, and $\varrho\to 0^+$, referred to as the \emph{light-load}
regime.

Given a stable policy $P$ the steady-state service delay for $\alpha$-demands is defined as
$D_{\alpha}(P):=\lim_{i \to +\infty} \expectation{D_{\alpha}(i)}$, and
the steady-state wait time for $\alpha$-demands is $W_{\alpha}(P):=D_{\alpha}(P) - \bar
s_{\alpha}$.  Thus, for a stable policy $P$, the \emph{average delay
  per demand} is
\[
D(P) = \frac{1}{\Lambda}\sum_{\alpha=1}^m \lambda_{\alpha}
D_{\alpha}(P),
\]
where $\Lambda:=\sum_{\alpha =1}^m\lambda_{\alpha}$.  The average
delay per demand is the standard cost functional for queueing systems
with multiple classes of demands. Notice that we can write $D(P)=
\sum_{\alpha=1}^m c_{\alpha} D_{\alpha}(P)$ with $c_{\alpha} =
\lambda_{\alpha}/\Lambda$. Thus, we can model priority among classes
by allowing any convex combination of $D_{1},\ldots,D_{m}$.  If
$c_\alpha > \lambda_{\alpha}/\Lambda$, then the delay of
$\alpha$-demands is being weighted more heavily than in the average
case.  Thus, the quantity $c_{\alpha} \Lambda/\lambda_{\alpha}$ gives
the priority of $\alpha$-demands compared to that given in the average
delay case. Without loss of generality we can assume that
priority classes are labeled so that
\begin{equation}
\label{eq:c_cond}
\frac{c_1}{\lambda_1} \geq \frac{c_2}{\lambda_2} \geq
\cdots \geq \frac{c_m}{\lambda_m},
\end{equation} 
implying that if $\alpha <\beta$ for some
$\alpha,\beta\in\{1,\ldots,m\}$, then the priority of $\alpha$-demands is at least as high as that of
$\beta$-demands.  With these definitions, we are now
ready to state our problem.
\begin{quote}
  \textbf{Problem Statement:} Let $\Pi$ be the set of all causal,
  stable and stationary policies for dynamic vehicle routing with
  priority classes. Given the coefficients $c_{\alpha}>0$,
  $\alpha\in\{1,\ldots,m\}$, with $\sum_{\alpha=1}^m c_{\alpha} = 1$,
  and satisfying equation~\eqref{eq:c_cond}, let $D(P) :=
  \sum_{\alpha=1}^m c_{\alpha} D_{\alpha}(P)$ be the cost of a policy $P
  \in \Pi$. Then, the problem is to determine a vehicle routing policy
  $P^*$, if one exists, such that
  \begin{equation}
  \label{eq:weighted_delay}
  D(P^*)= \inf_{P\in \Pi} D(P).
  \end{equation}
\end{quote}

We let $D^*$ denote the right-hand side of
equation~\eqref{eq:weighted_delay}.  A policy $P$ for which $D(P)/D^*$ is
bounded has a \emph{constant-factor guarantee}. If $\limsup_{\varrho \to
  1^-} D(P)/D^* = \kappa < +\infty$, then the policy $P$ has a
\emph{heavy-load constant-factor guarantee} of $\kappa$. In this paper we
focus on the heavy-load regime, and look for policies with a heavy-load
constant-factor guarantee that is \emph{independent} of the number of
vehicles, the arrival rates of demands, the on-site service times, and the
convex combination coefficients. 

\section{Lower Bound in Heavy Load}
\label{sec:lower_bd}

In this section we present two lower bounds on the delay in
Eq.~\eqref{eq:weighted_delay}. The first holds only in heavy load
(i.e., as $\varrho\to 1^-$), while the second (less tight) bound holds
for all~$\varrho$.

\begin{theorem}[Heavy load lower bound]
\label{thm:gen_lower_bd}
In heavy load ($\varrho \to 1^-$), for every routing policy $P$, 
\begin{equation}
  \label{eq:lower_bd}
  D(P) \geq \frac{\btsp^2|\env|}{2n^2v^2(1-\varrho)^2}\sum_{\alpha=1}^m\left(c_{\alpha}+2\sum_{j=\alpha+1}^{m}c_j\right)\lambda_{\alpha}.
\end{equation}
where $c_1,\ldots,c_m$ satisfy Eq.~(\ref{eq:c_cond}).
\end{theorem}
\newcommand{\remote}{r}
\begin{proof}
  Consider a tagged demand $i$ of type $\alpha$, and let us quantify
  its total service requirement.  The demand requires on-site service
  time $s_{\alpha}(i)$. Let us denote by $d_{\alpha}(i)$ the distance
  from the location of the demand served prior to $i$, to $i$'s
  location.  In order to compute a lower bound on the wait time, we
  will allow ``remote'' servicing of some of the demands.  For an
  $\alpha$-demand $i$ that can be serviced remotely, the travel
  distance $d_{\alpha}(i)$ is zero (i.e., a service vehicle can
  service the $i$th $\alpha$-demand from any location by simply
  stopping for the on-site service time $s_{\alpha}(i)$).  Thus, the
  wait time for the modified remote servicing problem provides a lower
  bound on the wait time for the problem of interest.  To formalize
  this idea, we introduce the variables $\remote_{\alpha}\in\{0,1\}$
  for each $\alpha\in\{1,\ldots,m\}$.  If $\remote_{\alpha} = 0$, then
  $\alpha$-demands can be serviced remotely.  If $\remote_{\alpha}=1$,
  then $\alpha$-demands must be serviced on location.  We assume that
  $r_{\alpha} = 1$ for at least one $\alpha\in\{1,\ldots,m\}$.  Thus,
  the total service requirement of $\alpha$-demand $i$ is
  $r_{\alpha}d_{\alpha}(i) + s_{\alpha}(i)$.  The steady-state
  expected service requirement is $\remote_{\alpha} \bar d_{\alpha} +
  s_{\alpha}$, where $\bar d_{\alpha} :=\lim_{i\to
    +\infty}\expectation{d_{\alpha}(i)}$.  In order to maintain
  stability of the system we must require
\begin{equation}
\label{eq:stability_cond}
\frac{1}{n}\sum_{\alpha =1}^m\lambda_{\alpha}\left(\frac{r_{\alpha}\bar d_{\alpha}}{v} + \bar s_{\alpha}\right) < 1.
\end{equation}
Applying the definition of $\varrho$ in Eq.~\eqref{eq:rho_def}, we
write Eq.~\eqref{eq:stability_cond} as
\begin{equation}
\label{eq:stability_rearranged}
\sum_{\alpha=1}^mr_{\alpha}\lambda_{\alpha}\bar d_{\alpha} < (1-\varrho)nv.
\end{equation}

For a stable policy $P$, let $\bar N_{\alpha}$ represent the
steady-state expected number of unserviced $\alpha$-demands. Then, the
expected total number of outstanding demands that require on-site service
(i.e., cannot be serviced remotely) is given by $\sum_{j=1}^mr_{j}\bar
N_{j}$. We now apply a result from the dynamic traveling repairperson
problem (see \cite{HX:95}, page $23$) which states that in heavy load
($\varrho \to 1^-$), if the steady-state number of outstanding demands
is $N$, then a lower bound on expected travel distance between demands
is $(\btsp/\sqrt{2})\sqrt{|\env|/N}$.  Applying this result we have
that
\begin{equation}
\label{eq:d_lower_bd}
\bar d_{\alpha} \geq \frac{\btsp}{\sqrt{2}}
\sqrt{\frac{|\env|}{\sum_{j}r_{j}\bar N_{j}}}=:\bar d,
\end{equation}
for each $\alpha\in\{1,\ldots,m\}$. Combining with
Eq.~(\ref{eq:stability_rearranged}), squaring both sides, and
rearranging we obtain
\[
\frac{\btsp^2}{2}
\frac{|\env|(\sum_{\alpha}r_{\alpha}\lambda_{\alpha})^2}{n^2v^2(1-\varrho)^2}
< \sum_{\alpha}r_{\alpha}\bar N_{\alpha}.
\]
From Little's law, $\bar N_{\alpha} = \lambda_{\alpha} W_{\alpha}$ for each
$\alpha\in\{1,\ldots,m\}$, and thus
\begin{equation}
\label{eq:constraint1}
\sum_{\alpha}r_{\alpha}\lambda_{\alpha}
  W_{\alpha}  >  
\frac{\btsp^2}{2} \frac{|\env|}{n^2v^2(1-\varrho)^2} 
\left(\sum_{\alpha}r_{\alpha}\lambda_{\alpha}\right)^2.
\end{equation}
Recalling that $W_{\alpha} = D_{\alpha} - \bar s_{\alpha}$ and
$\remote_{\alpha}\in\{0,1\}$ for each $\alpha \in \{1,\ldots,m\}$, we
see that Eq.~(\ref{eq:constraint1}) gives us $2^m - 1$
constraints on the feasible values of $D_{1}(P),\ldots,D_m(P)$.
Hence, a lower bound on $D^*$ can be found by minimizing $\sum_{\alpha
  = 1}^m W_{\alpha}$ subject to the constraints in
Eq.~\eqref{eq:constraint1}.  By considering the dual of this
problem, one can verify that under the class labeling in
Eq.~\eqref{eq:c_cond}, the problem is equivalent to:
\begin{align*}
&\text{\textbf{minimize}} \quad \;\sum_{\alpha=1}^mc_{\alpha} W_{\alpha}, \\
&\text{\textbf{subject to}} \\
&\begin{bmatrix}
  \lambda_1 & 0 & 0 &\cdots & 0 \\
  \lambda_1 & \lambda_2 & 0 &\cdots & 0  \\
  \vdots & \vdots & \ddots&  & 0 \\
  \lambda_1 & \lambda_2 & \lambda_3 & \cdots & \lambda_m \\
\end{bmatrix}
\begin{bmatrix}
W_{1} \\
W_2 \\
\vdots \\
W_{m}
\end{bmatrix} \geq 
\Psi\begin{bmatrix}
\lambda_{1}^2 \\
(\lambda_{1} +\lambda_2)^2 \\
\vdots \\
(\lambda_{1}+\cdots +\lambda_m)^2
\end{bmatrix},
\end{align*}
where
\[
\Psi := \frac{\btsp^2}{2} \frac{|\env|}{n^2v^2(1-\varrho)^2}.
\]
Under the class labeling in Eq.~\eqref{eq:c_cond} the above
linear program is feasible and bounded, and its solution
$(W_1^*,\ldots,W_m^*)$ is given by
\begin{equation*} 
W_{\alpha}^* = 
  \Psi\left(\lambda_{\alpha}+2\sum_{j=1}^{\alpha-1}\lambda_j\right).
\end{equation*} 
After rearranging, the optimal value of the cost function, and thus
the lower bound on $D^*$, is given by
\begin{align*}
\sum_{\alpha=1}^mc_{\alpha} W_{\alpha}^* 
=\Psi\sum_{\alpha=1}^m\left(c_{\alpha}+2\sum_{j=\alpha+1}^{m}c_j\right)
\lambda_{\alpha}.
\end{align*}
Applying the definition of $\Psi$ we obtain the desired result.
\end{proof}

\begin{remark}[Lower bound for all $\varrho\in{[0,1)}$]
  With slight modifications, it it possible to obtain a less tight
  lower bound valid for all values of $\varrho$. In the above
  derivation, the assumption that $\varrho \to 1^{-}$ is used in
  Eq.~\eqref{eq:d_lower_bd}. It is possible to use, instead, a lower
  bound valid for all $\varrho\in{[0,1)}$ (see \cite{DJS-GJvR:93a}):
\begin{equation*}
\bar d_{\alpha} \geq \gamma
\sqrt{\frac{|\env|}{\sum_{\alpha}r_{\alpha}N_{\alpha}  + n/2}},
\end{equation*}
where $\gamma = 2/(3\sqrt{2\pi}) \approx 0.266$.  Using this bound we
obtain the same linear program as in the proof of
Theorem~\ref{thm:gen_lower_bd}, with the difference that $\Psi$ is now a
function given by
\[
\Psi(x) := \frac{ \gamma^2|\env|}{n^2v^2(1-\varrho)^2}x  - \frac{n}{2}.
\]
Following the procedure in the proof of Theorem~\ref{thm:gen_lower_bd}
\begin{align*}
  W_1^* &= \frac{\gamma^2|\env|}{n^2v^2(1-\varrho)^2} \lambda_1 -
  \frac{n}{2\lambda_1} \\
  W_\alpha^* & = \frac{\gamma^2|\env|}{n^2v^2(1-\varrho)^2}
  \left(\lambda_{\alpha}+2\sum_{j=1}^{\alpha-1}\lambda_j\right),
\end{align*}
for each $\alpha \in\{2,\ldots,m\}$.
Finally, for every policy $P$, $D_{\alpha}(P) \geq W_{\alpha}^* +\bar
s_{\alpha}$, and thus
\begin{equation}
\label{eq:univ_lower_bd}
  D(P) \geq
  \frac{\gamma^2|\env|}{n^2v^2(1-\varrho)^2}\sum_{\alpha=1}^m\left(\left(c_{\alpha}+2\sum_{j=\alpha+1}^{m}c_j\right)\lambda_{\alpha}\right) 
  - \frac{nc}{2\lambda_1} +\sum_{\alpha=1}^mc_{\alpha}\bar s_{\alpha},
\end{equation}
for all $\varrho \in {[0,1)}$ under the labeling in
Eq.~(\ref{eq:c_cond}).  \oprocend
\end{remark}

\section{Separate Queues Policy}
\label{sec:SQ_policy}

In this section we introduce and analyze the Separate Queues (SQ)
policy.  We show that this policy is within a factor of $2m^2$ of the
lower bound in heavy load.

To present the SQ policy we need some notation. We assume vehicle
$k\in\{1,\ldots,n\}$ has a service region $R^{[k]} \subset \env$, such
that $\{R^{[1]},\ldots,R^{[n]}\}$ form a partition of the environment
$\env$.  In general the partition could be time varying, but for the
description of the SQ policy this will not be required.  We assume
that information on outstanding demands of type
$\alpha\in\{1,\ldots,m\}$ in region $R^{[k]}$ at time $t$ is
summarized as a finite set of demand positions $Q_{\alpha}^{[k]}(t)$
with $N_{\alpha}^{[k]}(t):=\card(Q_{\alpha}^{[k]}(t))$ .  Demands of
type $\alpha$ with location in $R^{[k]}$ are inserted in the set
$Q_{\alpha}^{[k]}$ as soon as they are generated. Removal from the set
$Q_{\alpha}^{[k]}$ requires that service vehicle $k$ moves to the
demand location, and provides the on-site service.  With this notation
the policy is given as Algorithm~1.

\begin{algorithm} 
  \dontprintsemicolon %
  \KwAssumes{A probability distribution
    $\mathbf{p}=[p_1,\ldots,p_m]$.}  %
  Partition $\env$ into $n$ equal area regions and assign one
  vehicle to each region. \;%
  \ForEach{\textup{vehicle-region pair $k$}}%
  {%
    \eIf{\textup{the set $\cup_{\alpha}Q_{\alpha}^{[k]}$ is empty}}%
    {%
      Move vehicle toward the median of its own region until a demand
      arrives. \;%
    }%
    { %
      Select $Q\in \{Q_{1}^{[k]},\ldots,Q_{m}^{[k]}\}$ according to
      $\mathbf{p}$. \; %
      \If{\textup{$Q$ is empty}} %
      { %
        Reselect until $Q$ is nonempty. \; %
      }%
      Compute TSP tour through all demands in $Q$. \; %
      Service $Q$ following the TSP tour, starting at the demand
      closest to the vehicle's current position. \; }%
    Repeat.\;%
  }Optimize over $\mathbf{p}$.
  \caption{\bf Separate Queues (SQ) Policy}
\end{algorithm}

\subsection{Stability Analysis of the SQ Policy in Heavy Load}

In this section we will analyze the SQ policy in heavy load, i.e., as
$\varrho \to 1^-$.  In the SQ policy each region $R^{[k]}$ has equal
area, and contains a single vehicle.  Thus, the $n$ vehicle problem in
a region of area $|\env|$ has been turned into $n$ independent single-vehicle problems, each in a region of area $|\env|/n$, with arrival
rates $\lambda_{\alpha}/n$.  To determine the performance of the
policy we need only study the performance in a single region $k$.  For
simplicity of notation we omit the label $k$.  We refer to the time
instant $t_i$ in which the vehicle computes a new $\TSP$ tour as the
epoch $i$ of the policy; we refer to the time interval between epoch
$i$ and epoch $i+1$ as the $i$th iteration and we will refer to its
length as $T_i$. Finally, let $N_{\alpha}(t_i):=N_{\alpha,i}$, $\alpha
\in \{1,\ldots,m\}$, be the number of outstanding $\alpha$-demands at
beginning of iteration $i$.

The following straightforward lemma, proved in
\cite{SLS-MP-FB-EF:08g}, will be essential in deriving our main
results.
\begin{lemma}[Number of outstanding demands]
  \label{lemma:nLarge}
  In heavy load (i.e., $\varrho \to 1^{-}$), after a transient, the
  number of demands serviced in a single tour of the vehicle in the SQ
  policy is very large with high probability (i.e., the number of
  demands tends to $+\infty$ with probability that tends to $1$, as
  $\varrho$ approaches $1^-$).
\end{lemma}

Let $TS_j$ be the event that $Q_{j}$ is selected for service at
iteration $i$ of the SQ policy. By the law of total probability
\begin{equation*}
\begin{split} 
\expectation{N_{\alpha,i+1}} = \sum_{j=1}^{m} p_j 
\condexpectation{N_{\alpha,i+1}}{TS_j}, \quad \alpha \in \{1,\ldots,m\},
\end{split}
\end{equation*}
where the conditioning is with respect to the task being performed during
iteration $i$. During iteration $i$ of the policy, demands arrive
according to independent Poisson processes. Call
$N_{\alpha,i}^{\text{new}}$ the $\alpha$-demands ($\alpha \in
\{1,\ldots,m\}$) newly arrived during iteration $i$; then, by
definition of the SQ policy
\begin{equation*}
\condexpectation{N_{\alpha,i+1}}{TS_j} = 
\begin{cases}
\condexpectation{N_{\alpha,i}^{\text{new}}}{TS_j}, &  \textrm{if $\alpha = j$}\\
\condexpectation{N_{\alpha,i}}{TS_j}+\condexpectation{N_{\alpha,i}^{\text{new}}}{TS_j}, &  \textrm{o.w.}
\end{cases}
\end{equation*} 

By the law of iterated expectation, we have
$\condexpectation{N_{\alpha,i}^{\text{new}}}{TS_j} = (\lambda_{\alpha}/n)
\condexpectation{T_i}{TS_j}$. Moreover, since the number of demands
outstanding at the beginning of iteration $i$ is independent of the
task that will be chosen, we have
$\condexpectation{N_{\alpha,i}}{TS_j} =
\expectation{N_{\alpha,i}}$. Thus we obtain
\begin{equation*}
\condexpectation{N_{\alpha,i+1}}{TS_j} = \begin{cases}
\frac{\lambda_{\alpha}}{n}  \condexpectation{T_i}{TS_j}, &  \textrm{if $\alpha = j$}\\
\expectation{N_{\alpha,i}}+\frac{\lambda_{\alpha}}{n}
\condexpectation{T_i}{TS_j}, &  \textrm{o.w.}
\end{cases}
\end{equation*} 

Therefore, we are left with computing the conditional expected values
of $T_i$.  The length of $T_i$ is given by the time needed by the
vehicle to travel along the TSP tour plus the time spent to service
demands. Assuming $i$ large enough, Lemma \eqref{lemma:nLarge} holds,
and we can apply Eq.~\eqref{eq:tspd} to estimate from the quantities
$N_{\alpha,i}$, $\alpha \in \{1,\ldots,m \}$, the length of the TSP
tour at iteration $i$. Conditioning on $TS_j$ (when only demands of
type $j$ are serviced), we have
\begin{equation*}
\begin{split} \condexpectation{T_i}{TS_j} &=\frac{ \btsp
\sqrt{|\env|/n}}{v} \, \condexpectation{\sqrt{N_{j,i}}} {TS_j}+
\condexpectation{\sum_{k=1}^{N_{j,i}} s_{j,k}}{TS_j} \\&\leq \frac{ \btsp
\sqrt{|\env|/n}}{v}\, \sqrt{\expectation{N_{j,i}}}
+ \expectation{ N_{j,i} }\bar{s}_{j},
\end{split}
\end{equation*} 
where we have: (i) applied Eq.~\eqref{eq:tspd}, (ii) applied Jensen's inequality for concave functions, in the form  $\expectation{\sqrt{X}} \leq \sqrt{\expectation{X}}$, (iii) removed the conditioning on $TS_j$, since the random variables $N_{\alpha,i}$ are independent from \emph{future} events, and in particular from the choice of the task at iteration $i$, and (iv)  used the \emph{crucial} fact that the on-site service times are independent from the number of outstanding demands.

Collecting the above results (and using the shorthand $\bar X$ to indicate $\expectation{X}$, where $X$ is any random variable), we have
\begin{equation}
\label{eq:N_SQP_dyn}
\begin{split} \bar N_{\alpha,i+1} \leq &(1- p_{\alpha}) \bar N_{\alpha,i}+ \sum_{j=1}^{m} p_j\frac{\lambda_{\alpha}}{n} \Biggl [\frac{ \btsp
\sqrt{|\env|/n}}{v} \,\sqrt{\bar N_{j,i}}
+ \bar N_{j,i} \bar{s}_{j}   \Biggr],
\end{split}
\end{equation}
for each $\alpha\in\{1,\ldots,m\}$. The $m$ inequalities above
describe a system of recursive relations that allows to find an upper bound on
$\bar{N}_{\alpha,i}$, $\alpha \in \{1,\ldots,m\}$. The following
theorem (see Appendix for its proof) bounds the values to which they converge.

\begin{theorem}[Queue length]
\label{thm:Queue_length_SQP}
In heavy load, for every set of initial conditions $\{\bar N_{\alpha,0}\}_{\alpha \in
  \{1,\ldots,m\}}$, the trajectories $i \mapsto \bar N_{\alpha,i}$,
$\alpha \in \{1,\ldots,m\}$, resulting from
Eqs.~\eqref{eq:N_SQP_dyn}, satisfy
  \[
  \limsup_{i\to+\infty}\bar N_{\alpha,i} \leq \frac{\btsp^2
    |\env|}{n^3v^2 (1-\varrho)^2}\frac{\lambda_{\alpha}}{p_{\alpha}}
  \left(\sum_{j=1}^m\sqrt{\lambda_j p_j}\right)^2.
\]
\end{theorem}

\subsection{Delay of the SQ Policy in Heavy Load}

From Theorem~\ref{thm:Queue_length_SQP}, and using Little's law, the
delay of $\alpha$-demands is 
\begin{align*}
D_{\alpha}(SQ) &\leq \frac{n}{\lambda_{\alpha}}
  \limsup_{i\to+\infty}\bar N_{\alpha,i} + \bar s_{\alpha} \\&=  \frac{\btsp^2
    |\env|}{n^2v^2 (1-\varrho)^2}\frac{1}{p_{\alpha}}
  \left(\sum_{j=1}^m\sqrt{\lambda_j p_j}\right)^2,
\end{align*}
where we neglected $\bar s_{\alpha}$ because of the heavy-load assumption.

Thus, the delay (as defined in Eq.~\eqref{eq:weighted_delay}) of the
SQ policy, satisfies in heavy load 
\begin{equation}
  \label{eq:SQ_delay}
  D(SQ) \leq\frac{\btsp^2 |\env|}{n^2v^2 (1-\varrho)^2}
  \sum_{\alpha=1}^m \frac{c_{\alpha}}{p_{\alpha}} \left(
    \sum_{i=1}^m\sqrt{\lambda_i p_i}\right)^2.
\end{equation}
With this expression we prove our main result on the performance of
the SQ policy.
\begin{theorem}[SQ policy performance]
\label{thm:SQ_perform}
In heavy load, the delay of the SQ policy is within a factor $2m^2$ of the
optimal, independent of the arrival rates
$\lambda_1,\ldots,\lambda_m$, coefficients $c_1,\ldots,c_m$, service
times $\bar s_1,\ldots,\bar s_m$, and the number of vehicles $n$.
\end{theorem}
\begin{proof}
We would like to compare the performance of this policy with the lower
bound.  To do this, consider setting
$p_{\alpha} := c_{\alpha}$ for each $\alpha\in\{1,\ldots,m\}$. Defining $B:=\btsp^2 |\env|/(n^2v^2 (1-\varrho)^2)$,
Eq.~\eqref{eq:SQ_delay} can be written as
\begin{align*}
  D(SQ) &\leq B m \left(\sum_{i=1}^m\sqrt{c_i\lambda_i}\right)^2.
\end{align*}
Next, the lower bound in Eq.~\eqref{eq:lower_bd} is
\begin{align*}
  D^* &\geq \frac{B}{2} \sum_{i=1}^m \left(c_{i} +
      2\sum_{j=i+1}^{m} c_j\right) \lambda_{i} \geq \frac{B}{2}
  \sum_{i=1}^m \left(c_{i}\lambda_{i}\right).
\end{align*}
Thus, comparing the upper and lower bounds
\begin{equation}
\label{eq:const_factor}
\frac{D(SQ)}{D^*} \leq 2m \frac{\left(\sum_{i=1}^m\sqrt{c_i \lambda_i}\right)^2}{\sum_{i=1}^m \left(c_{i}\lambda_{i}\right)}.
\end{equation}
Letting $x_i := \sqrt{c_i\lambda_i}$, and $\mathbf{x}
:=[x_1,\ldots,x_m]$, the numerator of the fraction in
Eq.~\eqref{eq:const_factor} is $\|\mathbf{x}\|_1^2$, and the
denominator is $\|\mathbf{x}\|_2^2$.  But the one- and two-norms of a
vector $\mathbf{x}\in\R^m$ satisfy $\|\mathbf{x}\|_1 \leq \sqrt{m}
\|\mathbf{x}\|_2$.  Thus, in heavy load we obtain

\[
\frac{D(SQ)}{D^*} \leq 2m
\left(\frac{\|\mathbf{x}\|_1}{\|\mathbf{x}\|_2}\right)^2\leq 2m^2,
\]
and the policy is a $2m^2$-factor approximation.
\end{proof}

\begin{remark}[Relation to RP policy in~\cite{SLS-MP-FB-EF:08g}]
  For $m=2$ the SQ policy is within a factor of $8$ of the optimal.
  This improves on the factor of $12$ obtained for the Randomized
  Priority (RP) policy in \cite{SLS-MP-FB-EF:08g}. However, it
  appears that the RP policy bound is not tight, since for two
  classes, simulations indicate it performs no worse than the SQ policy.
  \oprocend
\end{remark}

\section{Simulations and Discussion}
\label{sec:simu}


In this section we discuss, through the use of simulations, the
performance of the SQ policy with the probability assignment
$p_{\alpha} := c_{\alpha}$, for each $\alpha\in\{1,\ldots,m\}$. In
particular, we study (i) the tightness of the upper bound in
equation~\eqref{eq:SQ_delay}, (ii) conditions for which the gap between the
lower bound in equation~\eqref{eq:lower_bd} and the upper bound in
equation~\eqref{eq:SQ_delay} is maximized, (iii) the suboptimality of the
probability assignment $p_{\alpha} = c_{\alpha}$, and (iv) the
difference in performance between the SQ policy and a policy that
merges all classes together irrespective of priorities.
Simulations of the SQ policy were performed using {\ttfamily
  linkern}\footnote{The TSP solver {\ttfamily linkern} is freely available
  for academic research use at {\ttfamily
    http://www.tsp.gatech.edu/concorde.html}.} as a solver to generate
approximations to the optimal TSP tour.


\subsection{Tightness of the Upper Bound}
We consider one vehicle, four classes of demands, and several values
of the load factor $\varrho$.
For each value of $\varrho$ we perform 100 runs.  In each run we
uniformly randomly generate arrival rates
$\lambda_1,\ldots,\lambda_m$, convex combination coefficients
$c_1,\ldots,c_m$, and on-site service times $\bar s_1,\ldots,\bar
s_m$, and normalize the values such that the constraints
$\sum_{\alpha=1}^m \lambda_{\alpha} \bar s_\alpha = \varrho$ and
$\sum_{\alpha=1}^m c_{\alpha}=1$ are satisfied.  In each run we
iterate the SQ policy $4000$ times, and compute the steady-state
expected delay by considering the number of demands in the last $1000$
iterations. For each value of $\varrho$ we compute the ratio $\chi$
between the expected delay and the theoretical upper bound in
equation~(\ref{eq:SQ_delay}).  Table~\ref{tab:tightness} reports the ratio,
its standard deviation, and its minimum and maximum values for each
$\varrho$ value.
One can see that the upper bound provides a reasonable approximation
for load factors as low as $\varrho = 0.75$.
\begin{table}[bth]
  \centering \footnotesize
  \begin{tabular}{c|c|c|c|c}
    Load factor ($\varrho$) & $\expectation{\chi}$ &  $\sigma_{\chi}$ & $\max{\chi}$ & $\min{\chi}$\\
    \hline
    \hline
    0.75 &  0.803 & 0.092 & 1.093 & 0.354 \\
    0.8 &  0.778 & 0.108 & 0.943 & 0.256\\
    0.85 & 0.773 & 0.111 & 1.150 & 0.417\\
    0.9 & 0.733 & 0.159 & 1.162 & 0.203\\
    0.95  & 0.716 & 0.131 & 0.890 &  0.257
  \end{tabular}
  \caption{Ratio $\chi$ between experimental results and upper bound for
    various values of $\varrho$.} \label{tab:tightness}
\end{table}

\subsection{Unfavorable Conditions for the SQ Policy}

One may question if for some sets $\{ \lambda_\alpha \}$ and $\{
c_\alpha \}$, $\alpha \in \{1,\ldots,m\}$, the ratio between upper
bound \eqref{eq:SQ_delay} and lower bound \eqref{eq:lower_bd} is
indeed close to $2m^2$. The answer is affirmative: consider, e.g., the
case $\lambda_1 \ll \lambda_2 \ll \ldots \ll \lambda_m$ and $c_1 \gg
c_2 \gg \ldots \gg c_m$, with $\lambda_{\alpha}c_{\alpha} = a$, for
some positive constant $a$. Then, the upper bound is equal to $Bm^3a$
and the lower bound is approximately equal to $Bma/2$, thus their
ratio is (arbitrarily) close to $2m^2$. Then, we simulated the SQ
policy for the case $\lambda_m = a\lambda_{m-1} =
a^2\lambda_{m-1}=\ldots=a^{m-1} \lambda_1$ and $c_1 = a c_2 =
\ldots=a^{m-1} c_m$ with $a=2$. Fig.~\ref{fig:worstCaseBound.pdf}
shows that the experimental value of the cost function (averaged over
$10$ simulation runs) indeed increases proportionally to $m^2$.

\begin{figure}
\centering
\includegraphics[width=0.8\linewidth]{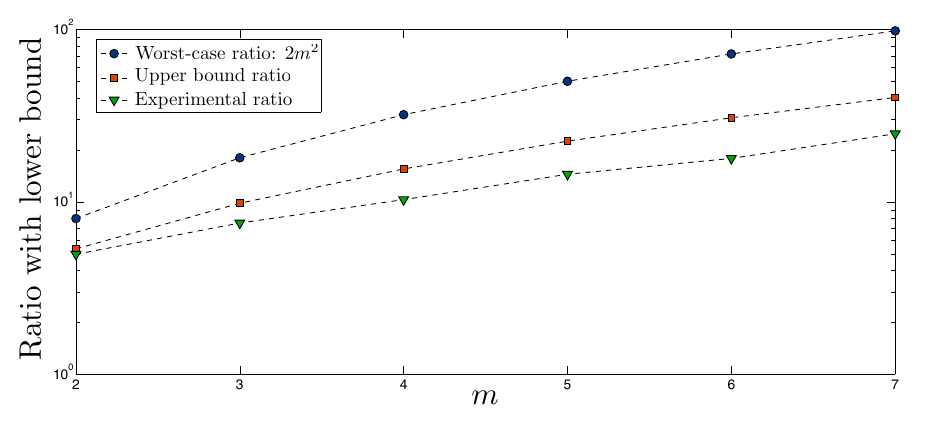}
\caption{Experimental results for the SQ policy in worst-case
  conditions; $\varrho = 0.85$ and $\lambda_1 = 1$.}
\label{fig:worstCaseBound.pdf}
\end{figure}

\subsection{Suboptimality of the Approximate Probability Assignment}
To prove Theorem~\ref{thm:SQ_perform} we used the probability
assignment
\begin{equation}
\label{eq:p_assign}
p_{\alpha} := c_{\alpha} \quad\text{for each $\alpha\in\{1,\ldots,m\}$}.
\end{equation}
However, one would like to select $[p_1,\ldots,p_m]=:\mathbf{p}$ that
minimizes the right-hand side of Eq.~\eqref{eq:SQ_delay}.  The
minimization of the right-hand side of Eq.~\eqref{eq:SQ_delay} is
a constrained multi-variable nonlinear optimization problem over
$\mathbf{p}$, that is, in $m$ dimensions.  However, for two classes of
demands the optimization is over a single variable $p_1$, and it can
be readily solved.  A comparison of optimized upper bound, denoted
$\upbd_{\opt}$, with the upper bound obtained using the probability
assignment in Eq.~\eqref{eq:p_assign}, denoted $\upbd_{c}$, is
shown in Fig.~\ref{fig:opt_p}.
\begin{figure}
\centering
\includegraphics[width=0.8\linewidth]{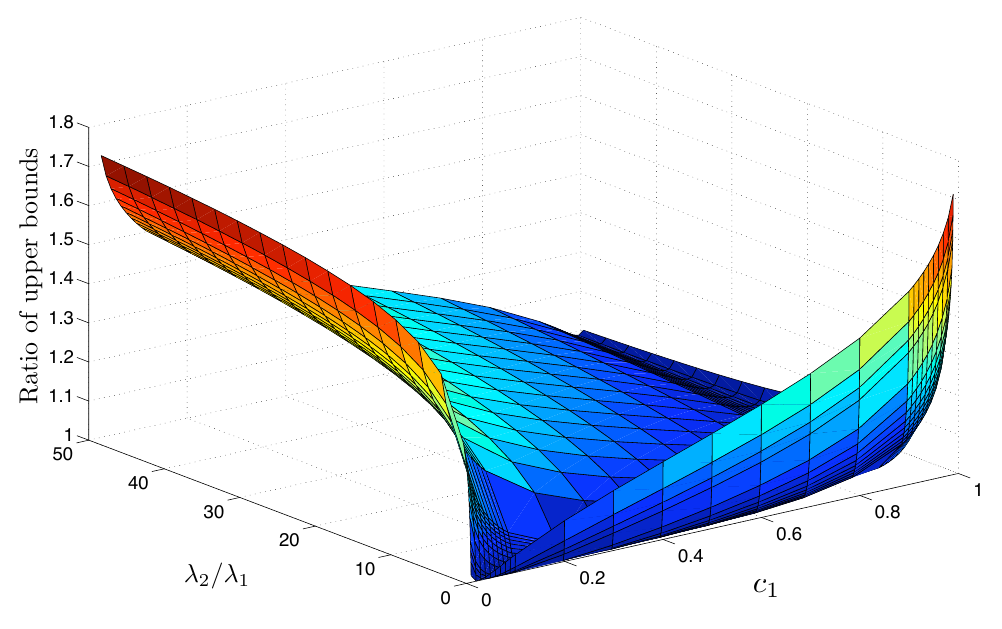}
\caption{The ratios $\upbd_c/\upbd_{opt}$ for 2 classes of demands.}
\label{fig:opt_p}
\end{figure}

For $m>2$ we approximate the solution of the optimization problem as
follows.  For each value of $m$ we perform 1000 runs.  In each run we
randomly generate $\lambda_1,\ldots,\lambda_m$, $c_1,\ldots,c_m$, and
five sets of initial probability assignments
$\mathbf{p}_1,\ldots,\mathbf{p}_5$.  From each initial probability
assignment we use a line search to locally optimize the probability
assignment.  We take the ratio between $\upbd_c$ and the least upper
bound $\upbd_{\text{local opt}}$ obtained from the five locally
optimized probability assignments.  We also record the maximum
variation in the five locally optimized upper bounds.  This is
summarized in Table~\ref{tab:max_dev}.
\begin{table}
  \centering
  \begin{tabular}{c|c|c}
    Number of classes ($m$) & $\max \upbd_{c}/\upbd_{\opt}$ &  
    Max. $\%$ variation \\
    \hline
    \hline
    3 & 1.60 & 0.12 \\
    4 & 1.51 & 0.04 \\
    5 & 1.51 & 0.08 \\
    6 & 1.74 & 0.02 \\
    7 & 1.88 & 0.08 \\
    8 & 1.63 & 0.15
  \end{tabular}
  \caption{Ratio of upper bound with $p_{\alpha} = c_{\alpha}$ 
    and upper bound with optimized $\mathbf{p}$.}
  \label{tab:max_dev}
\end{table}
The second column shows the largest ratio obtained over the 1000 runs.
The third column shows the largest $\%$ variation in the 1000 runs.
The assignment in Eq.~\eqref{eq:p_assign} performs within a factor of
two of the optimized assignment.  In addition, the optimization
appears to converge to values close to a global optimum since all five
random conditions converge to values that are within $\sim 0.1\%$ of
each other on every run.

\subsection{The Merge Policy}
The simplest possible policy for our problem would be to ignore
priorities and service demands all together, by repeatedly forming TSP
tours of outstanding demands (i.e., by using the SQ policy as though
there were only one class). We call such a policy the Merge policy.
However, the performance of the SQ and the Merge policy can be
arbitrarily far apart.  Indeed, by defining the overall arrival rate
$\Lambda := \sum_{\alpha=1}^m \lambda_\alpha$ and overall mean on-site
service $\bar S:=\sum_{\alpha=1}^m \lambda_\alpha$, and by using the
upper bounds in \cite{DJS-GJvR:91}, we immediately obtain as an upper
bound for the Merge policy: $D(\mathrm{Merge})\leq \frac{\btsp^2
  |\env|\Lambda}{n^2v^2 (1-\varrho)^2}$. Then, we see that
$D(\mathrm{Merge})/D(SQ)$ can be arbitrarily large by choosing
$\lambda_m \gg \lambda_\alpha$ and $c_m \ll c_\alpha$, with $\alpha
\in \{1,\ldots,m-1\}$. This behavior is confirmed by experimental
results, as depicted in Fig. \ref{fig:MPbad} where we show the
experimental ratios of delays between Merge and SQ policy (the ratios
are averaged values over 10 simulation runs).

\begin{figure}
\centering
\includegraphics[width=0.8\linewidth]{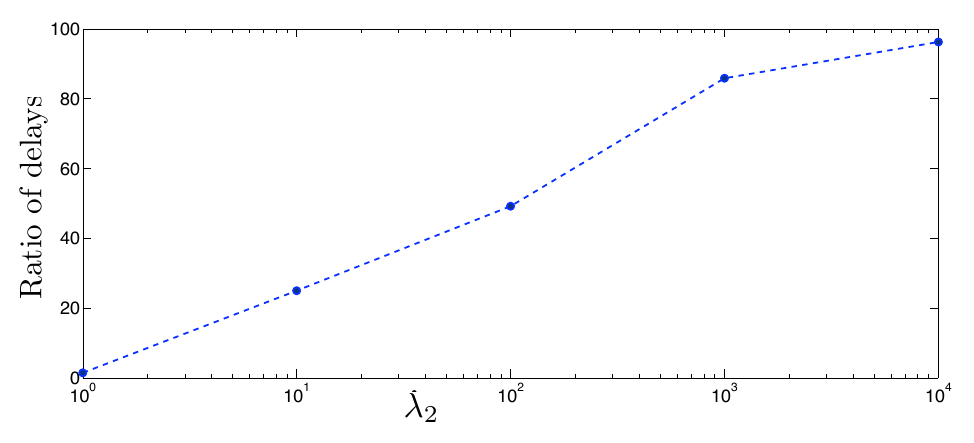}
\caption{Ratio of experimental delays between Merge policy and SQ
  policy as a function of $\lambda_2$, with $m=2$, $\lambda_1=1$, c =
  0.995 and $\varrho = 0.9$.}
\label{fig:MPbad}
\end{figure}

\section{Conclusions}
\label{sec:conc}
In this paper we studied a dynamic multi-vehicle routing problem with
multiple classes of demands.  For every set of coefficients, we
determined a lower bound on the achievable convex combination of the
class delays.  We presented the Separate Queues (SQ) policy and showed
that its deviation from the lower bound depends only on the number of
the classes. We believe that there is room for improvement in the
lower bound, and thus the SQ policy's performance may be significantly
better than is indicated by its deviation from the current lower
bound.  Thus, our main thrust of future work will be in trying to
raise the lower bound.  We are also interested in combining the
aspects of multi-class vehicle routing with problems in which demands
require teams of vehicles for their service, and in extending our results to the case of non-uniform demand densities (possibly class dependent).


\bibliographystyle{IEEEtran}
\bibliography{brevalias,New,Main,FB}

\begin{thebibliography}{10}
\providecommand{\url}[1]{#1}
\csname url@samestyle\endcsname
\providecommand{\newblock}{\relax}
\providecommand{\bibinfo}[2]{#2}
\providecommand{\BIBentrySTDinterwordspacing}{\spaceskip=0pt\relax}
\providecommand{\BIBentryALTinterwordstretchfactor}{4}
\providecommand{\BIBentryALTinterwordspacing}{\spaceskip=\fontdimen2\font plus
\BIBentryALTinterwordstretchfactor\fontdimen3\font minus
  \fontdimen4\font\relax}
\providecommand{\BIBforeignlanguage}[2]{{%
\expandafter\ifx\csname l@#1\endcsname\relax
\typeout{** WARNING: IEEEtran.bst: No hyphenation pattern has been}%
\typeout{** loaded for the language `#1'. Using the pattern for}%
\typeout{** the default language instead.}%
\else
\language=\csname l@#1\endcsname
\fi
#2}}
\providecommand{\BIBdecl}{\relax}
\BIBdecl

\bibitem{LK:76}
L.~Kleinrock, \emph{Queueing Systems. Volume II: Computer Applications}.\hskip
  1em plus 0.5em minus 0.4em\relax New York, NY: John Wiley and Sons, 1976.

\bibitem{EGC-IM:80}
E.~G. {Coffman~Jr}. and I.~Mitrani, ``A characterization of waiting time
  performance realizable by single-server queues,'' \emph{Operations Research},
  vol.~28, no.~3, pp. 810--821, 1980.

\bibitem{DB-ICP-JNP:94}
D.~Bertsimas, I.~C. Paschalidis, and J.~N. Tsitsiklis, ``Optimization of
  multiclass queueing networks: Polyhedral and nonlinear characterizations of
  achievable performance,'' \emph{The Annals of Applied Probability}, vol.~4,
  no.~1, pp. 43--75, 1994.

\bibitem{DJS-GJvR:91}
D.~J. Bertsimas and G.~J. {van~Ryzin}, ``A stochastic and dynamic vehicle
  routing problem in the {E}uclidean plane,'' \emph{Operations Research},
  vol.~39, pp. 601--615, 1991.

\bibitem{DJS-GJvR:93a}
------, ``Stochastic and dynamic vehicle routing in the {E}uclidean plane with
  multiple capacitated vehicles,'' \emph{Operations Research}, vol.~41, no.~1,
  pp. 60--76, 1993.

\bibitem{DJS-GJvR:93b}
------, ``Stochastic and dynamic vehicle routing with general interarrival and
  service time distributions,'' \emph{Advances in Applied Probability},
  vol.~25, pp. 947--978, 1993.

\bibitem{EF-FB:03r}
E.~Frazzoli and F.~Bullo, ``Decentralized algorithms for vehicle routing in a
  stochastic time-varying environment,'' in \emph{Proc {CDC}}, Paradise Island,
  Bahamas, Dec. 2004, pp. 3357--3363.

\bibitem{MP-EF-FB:07g}
M.~Pavone, E.~Frazzoli, and F.~Bullo, ``Decentralized algorithms for stochastic
  and dynamic vehicle routing with general target distribution,'' in \emph{Proc
  {CDC}}, New Orleans, LA, Dec. 2007, pp. 4869--4874.

\bibitem{RCL-ARO:81}
R.~C. Larson and A.~R. Odoni, \emph{Urban Operations Research}.\hskip 1em plus
  0.5em minus 0.4em\relax Prentice Hall, 1981.

\bibitem{SLS-MP-FB-EF:08g}
S.~L. Smith, M.~Pavone, F.~Bullo, and E.~Frazzoli, ``Dynamic vehicle routing
  with heterogeneous demands,'' in \emph{Proc {CDC}}, Canc\'un, M\'exico, Dec.
  2008, pp. 1206--1211.

\bibitem{JMS:90}
J.~M. Steele, ``Probabilistic and worst case analyses of classical problems of
  combinatorial optimization in {E}uclidean space,'' \emph{Mathematics of
  Operations Research}, vol.~15, no.~4, p. 749, 1990.

\bibitem{GP-OCM:96}
G.~Percus and O.~C. Martin, ``Finite size and dimensional dependence of the
  {E}uclidean traveling salesman problem,'' \emph{Physical Review Letters},
  vol.~76, no.~8, pp. 1188--1191, 1996.

\bibitem{HX:95}
H.~Xu, ``Optimal policies for stochastic and dynamic vehicle routing
  problems,'' {Dept. of Civil and Environmental Engineering}, Massachusetts
  Institute of Technology, Cambridge, MA, 1995.

\end{thebibliography}

\section*{Appendix}

In this appendix we prove Theorem \ref{thm:Queue_length_SQP}.
Henceforth, we consider the relation ``$\leq$'' in $\reals^m$ as the
product order of $m$ copies of $\reals$ (in other words, given two
vectors $v, \, w \in \reals^m$, the relation $v \leq w$ is interpreted component-wise).

\begin{proof}[Proof of Theorem~\ref{thm:Queue_length_SQP}]
  Define $q_j := 1-p_j$ and let $\hat \lambda_{\alpha}$ denote the
  arrival rate in region $R^{[k]}$.  Thus $\hat \lambda_{\alpha} :=
  \lambda_{\alpha}/n$ for each $\alpha\in\{1,\ldots,m\}$.  Let
  $x(i):=( \bar N_{1,i},\bar N_{2,i},\ldots, \bar N_{m,i}) \in
  \reals^m$ and define two matrices
    \[
  A := \begin{bmatrix}
    \hat \lambda_{1} p_1\bar{s}_1 + q_1  & \hat \lambda_{1} p_2 \bar{s}_2 & \ldots &\hat \lambda_{1} p_m \bar{s}_m\\
    \hat \lambda_{2} p_1 \bar{s}_1 & \hat \lambda_{2} p_2 \bar{s}_2 + q_2 & \ldots &\hat \lambda_{2} p_m \bar{s}_m\\
    \vdots & &\ddots&\vdots\\
    \hat \lambda_{m} p_1 \bar{s}_1 & \hat \lambda_{m} p_2 \bar{s}_2 & \ldots &\hat \lambda_{m} p_m \bar{s}_m+ q_m \\
    \end{bmatrix},
  \]
  and
     \[
  B :=  \frac{ \btsp \sqrt{|\env|}}{\sqrt{n}v} \begin{bmatrix}
    \hat \lambda_{1} p_1  & \hat \lambda_{1} p_2 & \ldots &\hat \lambda_{1} p_m \\
    \hat \lambda_{2} p_1& \hat \lambda_{2} p_2 & \ldots &\hat \lambda_{2} p_m\\
    \vdots & &\ddots&\vdots\\
    \hat \lambda_{m} p_1 & \hat \lambda_{m} p_2& \ldots &\hat \lambda_{m} p_m  \\
    \end{bmatrix},
  \]
  Then Eqs.~\eqref{eq:N_SQP_dyn} can be written as
  \begin{equation}
  \label{eq:n_evol}
     x(i+1) \leq  A x(i) + B 
  \begin{bmatrix}
        \sqrt{ x_1(i)} \\ \sqrt{ x_2(i)} \\ \vdots \\ \sqrt{ x_m(i)}
  \end{bmatrix} =: f(x(i))
  \end{equation}
  where $f: \reals_{\geq0} \mapsto \reals_{\geq0}$, and $ x_j(i)$,
  $j\in\{1,\ldots,m\}$, are the components of vector $x(i)$. We refer
  to the discrete system in Eq.~\eqref{eq:n_evol} as System-X. Next we
  define two auxiliary systems, System-Y and System-Z. We define
  System-Y as
\begin{equation}
  \label{eq:barn_evol}
  y(i+1) = f( y(i)) .
  \end{equation}
  System-Y is, therefore, equal to System-X, with the exception that
  we replaced the inequality with an equality.

  Pick, now, any $\eps>0$. From Young's inequality
\begin{equation}\label{eq:ineq}
  \sqrt{a} \leq 
  \frac{1}{4\eps} + \eps a,
  \quad \text{for all } a\in\real_{\geq0}.
\end{equation}
Hence, for $i\mapsto y(i)\in\real_{\geq0}^m$, the
Eq.~\eqref{eq:barn_evol} becomes
\begin{align*}
  y(i+1) &\leq A y(i) + B \Big( \frac{1}{4\eps}
  \mathbf{1}_m
  + \eps   \,  y(i) \Big)
  \\
  &=\Big(A +\eps B \Big) y(i) + 
  \frac{1}{4\eps} B \mathbf{1}_m.
\end{align*}
where $\mathbf{1}_m$ is the vector $(1,1,\ldots,1)^{\text{T}}\in \reals^m$.
Next, define System-Z as 
\begin{equation}
  \label{eq:syst_z}
  z(i+1) =\Big(A +\eps B \Big) z(i) + 
  \frac{1}{4\eps} B \mathbf{1}_m =: g(z(i)).
\end{equation}

The proof now proceeds as follows. First, we show that if $x(0) =y(0)=
z(0)$, then
 \begin{equation}\label{eq:traj_bound}
  x(i) \leq y(i) \leq z(i), \quad \text{for all $i \geq 0$}
  \end{equation}
  Second, we show that the trajectories of System-Z are bounded;
  this fact, together with Eq.~\eqref{eq:traj_bound}, implies that
  also trajectories of System-Y and System-X are bounded. Third, and
  last, we will compute $\limsup_{i \to +\infty} y(i)$; this
  quantity, together with Eq.~\eqref{eq:traj_bound}, will yield the
  desired result.

  Let us consider the first issue. We have $y(1) = f(y(0))$ and $z(1)
  = g(z(0))$. Since, by assumption $z(0) =y(0)$, we have that $g(z(0))
  = g(y(0)) \geq f(y(0))$, where the last inequality follows from
  Eq.~\eqref{eq:ineq} and by definition of $f$ and $g$ . Therefore, we
  get $y(1) \leq z(1)$. Then, we have $y(2) = f(y(1))$ and $z(2) =
  g(z(1))$. Since $z(1), y(1) \in \reals^m_{\geq 0}$, and the elements
  in matrices $A$ and $B$ are all non-negative, then $y(1) \leq z(1)$
  implies $g(y(1)) \leq g(z(1))$. Using same arguments as before, we
  can write $z(2) \geq g(y(1)) \geq f(y(1)) = x(2)$; therefore, we get
  $y(2) \leq z(2)$. Then, it is immediate by induction that $y(i) \leq
  z(i)$ for all $i \geq 0$.
  
  Similarly, we have $x(1) \leq f(x(0)) = f(y(0)) = y(1)$, where we
  have used the assumption $x(0) = y(0)$. Then, we get $x(1) \leq
  y(1)$. Since $ x(1), y(1) \in \reals^m_{\geq 0}$, the elements in
  matrices $A$ and $B$ are nonnegative, and by the monotonicity of
  $\sqrt{\cdot}$, then $x(1) \leq y(1)$ implies $f(x(1)) \leq
  f(y(1))$. Therefore, we can write $x(2)\leq f(x(1)) \leq
  f(y(1))=y(2)$; thus, we get $x(2) \leq y(2)$. Then, it is immediate
  to show by induction that $ x(i) \leq y(i)$ for all $i \geq 0$, and
  Eq.~\eqref{eq:traj_bound} holds.
 
  We now turn our attention to the second issue, namely boundedness of
  trajectories for System-Z (in Eq.~\eqref{eq:syst_z}). Notice
  that System-Z is a discrete-time linear system. The eigenvalues of
  $A$ are characterized in the following lemma.
\begin{lemma}\label{lemma:real_eig_A}
  The eigenvalues of $A$ are real and with magnitude strictly less
  than $1$ (i.e., $A$ is a stable matrix).
\end{lemma}  
\begin{proof}
  Let $w \in \mathbb{C}^m$ be an eigenvector of $A$, and $\mu \in
  \mathbb{C}$ be the corresponding eigenvalue. Then we have $Aw~=~\mu
  w$. Define $r:=(p_1\bar{s}_1, p_2 \bar{s}_2,\ldots, p_m
  \bar{s}_m)$. Then the $m$ eigenvalue equations are
\begin{equation}
\label{SQP_eig_detailed}
\hat \lambda_j\,  w \cdot r + q_jw_j = \mu \, w_j, \quad j\in\{1,\ldots,m\},
\end{equation}
%
%
where $w \cdot r$ is the scalar product of vectors $w$ and $r$, and
$w_j$ is the $j$th component of $w$.

There are two possible cases. If $w \cdot r = 0$, then
Eq.~\eqref{SQP_eig_detailed} becomes $q_j\,w_j = \mu \, w_j$, for all
$j$. Since $w \neq 0$, there exists $j^*$ such that $w_j^* \neq 0$;
thus, we have $\mu = q_{j^*}$. Since $q_{j^*} \in \reals$ and $0 <
q_{j^*} <1$, we have that $\mu$ is real and $|\mu|<1$.

Assume, now, that $w \cdot r \neq 0$. This implies that $\mu \neq q_j$
and $w_j \neq 0$ for all $j$, thus we can write for all $j$
\begin{equation}\label{SQP_eig_detailed_2}
w_j = \frac{\hat \lambda_j}{\mu - q_j} \, w \cdot r
\end{equation}
Therefore 
\[
w_j = \frac{\hat \lambda_j}{\hat \lambda_1}\frac{\mu - q_1}{\mu - q_j}w_1. 
\]
Therefore, \eqref{SQP_eig_detailed_2} can be rewritten as
\begin{equation}\label{eq:SQP_eig_relation}
\sum_{j=1}^m \frac{r_j \hat \lambda_j}{\mu - q_j} = 1.
\end{equation}
Eq.~\eqref{eq:SQP_eig_relation} implies that the eigenvalues are
real. To see this, write $\mu = a + ib$, where $i$ is the imaginary
unit: then
\[
\sum_{j=1}^m \frac{r_j \hat \lambda_j}{a+ib - q_j} = \sum_{j=1}^m \frac{r_j
  \hat \lambda_j [(a-q_j) -ib]}{(a - q_j)^2 + b^2}
\]
Thus Eq.~\eqref{eq:SQP_eig_relation} implies
\[
 b\,\underbrace{\sum_{j=1}^m \frac{r_j \hat \lambda_j}{(a - q_j)^2 + b^2}}_{>0} = 0
\]
that is, $b = 0$. Eq.~\eqref{eq:SQP_eig_relation} also implies
that the eigenvalues (that are real) have magnitude strictly less than
$1$. Indeed, assume, by contradiction, that $\mu\geq 1$, then we would
have $\mu - q_j \geq 1 - q_j>0$ (recall that the eigenvalues are real
and $0<q_j<1$) and we could write
\[
\sum_{j=1}^m \frac{r_j \hat \lambda_j}{\mu - q_j} \leq \sum_{j=1}^m
\frac{r_j \hat \lambda_j}{1 - q_j} =\sum_{j=1}^m \bar s_j \hat \lambda_j =
\varrho < 1,
\]
and we get a contradiction. Assume, again by contradiction, that $\mu
\leq -1$, then we would trivially get another contradiction
$\sum_{j=1}^m r_j \hat \lambda_j/(\mu - q_j) < 0$, since $\mu - q_j <
0$. 
\end{proof}

Hence, $A\in\real^{m\times{m}}$ has eigenvalues strictly inside the
unit disk, and since the eigenvalues of a matrix depend continuously
on the matrix entries, there exists a sufficiently small $\eps>0$ such
that the matrix $A +\eps B $ has eigenvalues strictly inside the unit
disk.  Accordingly, each solution $i\mapsto
z(i)\in\real_{\geq0}^m$ of System-Z converges exponentially
fast to the unique equilibrium point
\begin{equation}\label{eq:fixed_point}
  z^* = 
  \Big(I_m - A -\eps B\Big)^{-1}
  \frac{1}{4\eps} B\mathbf{1}_m.
\end{equation}
Combining Eq.~\eqref{eq:traj_bound} with the previous statement, we
see that the solutions $i\mapsto x(i)$ and $i\mapsto y(i)$ are
bounded.  Thus
  \begin{equation}\label{eq:sup_ineq}
  \limsup_{i\to+\infty}x(i) \leq \limsup_{i\to+\infty}y(i)  < +\infty.
  \end{equation}
 
 Finally, we turn our attention to the third issue, namely the
  computation of $y := \limsup_{i \to +\infty} y(i)$. Taking the
  $\limsup$ of the left- and right-hand sides of
  Eq.~\eqref{eq:barn_evol}, and noting that
  \[
  \limsup_{i\to +\infty} \sqrt{y_{\alpha}(i)} = \sqrt{\limsup_{i\to
      +\infty} y_{\alpha}(i)} \quad \text{for
    ${\alpha}\in\{1,2,\ldots,m\}$},
  \]
  since $\sqrt{\cdot}$ is continuous and strictly monotone increasing
  on $\R_{>0}$, we obtain that
\[
y_{\alpha} = (1-p_{\alpha})y_{\alpha} + \hat \lambda_{\alpha}\,\sum_{j=1}^m
p_j \biggl (\frac{ \btsp \sqrt{|\env|}}{\sqrt{n}v} \sqrt{y_j} + \bar
s_j y_j \biggr).
\]
Rearranging we obtain
\begin{equation}
\label{eq:n_alpha_length}
p_{\alpha} y_{\alpha} =  \hat \lambda_{\alpha}\,\sum_{j=1}^m p_j \biggl (\frac{ \btsp \sqrt{|\env|}}{\sqrt{n}v} \sqrt{y_j} + \bar s_j y_j   \biggr).
\end{equation}
Dividing $p_{\alpha} y_{\alpha}$ by $p_1y_1$ we obtain
\begin{equation}
\label{eq:n_alpha_to_n_1}
y_{\alpha} = \frac{\hat \lambda_{\alpha} p_1}{\hat \lambda_1p_{\alpha}}y_1.
\end{equation}
Combining Eqs.~(\ref{eq:n_alpha_length}) and
(\ref{eq:n_alpha_to_n_1}), we obtain
\begin{align*}
  p_1 y_1 &= \varrho \, p_1 y_1 + \frac{ \btsp
    \sqrt{|\env|}}{\sqrt{n}v} \sqrt{p_1 \hat \lambda_1 y_1}
  \sum_{j=1}^m\sqrt{\hat \lambda_j p_j}
 \end{align*}
 Thus, recalling that $\hat \lambda_{\alpha}=\lambda_{\alpha}/n$, we
 obtain
\[
y_{\alpha} =\frac{ \btsp^2 |\env|}{n^3v^2 (1-\varrho^2)}
\frac{\lambda_{\alpha}}{p_{\alpha}}
\left(\sum_{j=1}^m\sqrt{\lambda_jp_j}\right)^2.
\]
Noting that from Eq.~\eqref{eq:sup_ineq}, $\limsup_{i\to +\infty}
N_{\alpha,i} \leq y_{\alpha}$, we obtain the desired result.
\end{proof}  

\end{document}